\documentclass[conference]{IEEEtran}

\usepackage[utf8]{inputenc} 
\usepackage[T1]{fontenc}
\usepackage{url}
\usepackage{ifthen}
\usepackage{cite}
\usepackage[cmex10]{amsmath} % Use the [cmex10] option to ensure complicance
\usepackage{epsfig} % for postscript graphics files
\usepackage{mathptmx} % assumes new font selection scheme installed
\usepackage{times} % assumes new font selection scheme installed
\usepackage{amsmath,amsthm} % assumes amsmath package installed
\usepackage{amssymb} % assumes amsmath package installed
\usepackage{url}
\usepackage{color}
\usepackage{soul}
\usepackage[]{algorithm2e}
\usepackage{enumerate}
\usepackage{mathrsfs}
\usepackage{comment}

\newtheorem{theorem}{Theorem}
\newtheorem{corollary}{Corollary}
\newtheorem{lemma}{Lemma}
\newtheorem{definition}{Definition}

\newtheorem{remark}{Remark}
\newtheorem{proposition}{Proposition}

\newcommand{\lp}{\left(}
\newcommand{\rp}{\right)}

\newcommand{\argmin}{\ensuremath{ \text{argmin} }}
\newcommand{\rank}{\ensuremath{ \text{rank} }}

\allowdisplaybreaks

\begin{document}
\title{Privacy-Utility Trade-off of Linear Regression under Random Projections and Additive Noise} 

%%% Several authors with up to three affiliations:
\author{%
  \IEEEauthorblockN{Mehrdad Showkatbakhsh} 
  \IEEEauthorblockA{UCLA, Los Angeles, CA\\
                    \textsf{mehrdadsh@ucla.edu}}
                    \and
  \IEEEauthorblockN{Can Karakus} 
  \IEEEauthorblockA{UCLA, Los Angeles, CA\\
                    \textsf{karakus@ucla.edu}}
                    \and
                      \IEEEauthorblockN{Suhas Diggavi} 
  \IEEEauthorblockA{UCLA, Los Angeles, CA\\
                    \textsf{suhasdiggavi@ucla.edu}}
}
\maketitle

\begin{abstract}
Data privacy is an important concern in machine learning,
and is fundamentally at odds with the task of training useful learning
models, which typically require acquisition of large amounts of
private user data. One possible way of fulfilling the machine learning
task while preserving user privacy is to train the model on a
transformed, noisy version of the data, which does not reveal the data
itself directly to the training procedure. In this work, we analyze
the privacy-utility trade-off of two such schemes for the problem of
linear regression: additive noise, and random projections. In contrast to previous work, we consider
a recently proposed notion of differential privacy that is based on
conditional mutual information (MI-DP), which is stronger than the conventional
$(\epsilon, \delta)$-differential privacy, and use relative objective
error as the utility metric. We find that projecting the data to a
lower-dimensional subspace before adding noise attains a better
trade-off in general. We also make a connection between privacy problem and (non-coherent) {\color{black} SIMO}, which has been extensively studied
in wireless communication, and use tools from there for the analysis. We
present numerical results demonstrating the performance of the
schemes.
\end{abstract}

%% The paper must be self-contained. However, if you are referring to
%% a full version for checking certain proofs, please provide the
%% publically accessible location below.  If the paper is completely
%% self-contained, you can remove the following line from your
%% submission.
% \textit{A full version of this paper is accessible at: \url{http://www.isit2018.org/}}

\section{Introduction}

% general privacy motivation, fundamental tension between privacy and ML task, in this paper we do ...
%With the widespread deployment of machine learning algorithms in popular online services, privacy of user data became a significant issue. On one hand,
High-complexity models are needed to solve modern learning problems, which require large amounts of data to achieve low generalization error. However, acquiring such data from users directly compromises the user privacy. Training useful machine learning models without compromising user privacy is an important and challenging research problem. One natural way to tackle this problem is to keep the data itself private, and reveal only a processed, noisy version of the data to the training procedure. Ideally, such processing would completely hide the content of the data samples, while still providing useful information to the training objective. In this paper, we analyze the privacy-utility trade-off of two such schemes for the linear regression problem: additive noise, where training is performed on the data samples with additive Gaussian noise; and random projections, where each data sample is randomly projected to a lower-dimensional subspace through Johnson-Lindenstrauss Transform (JLT) \cite{Vempala2005} before adding {\color{black} Gaussian} noise. We explore guarantees for a model that is trained on such transformed data for a given privacy constraint.

%A well-known privacy measure is given by the notion of differential privacy \cite{Dwork2006}.
Differential privacy is perhaps the most well-known notion for privacy \cite{Dwork2014}, and has been applied to a variety of domains (we refer reader to \cite{Sarwate2013} and \cite{Dwork2014} and references therein). It assumes a strong adversary which has access to all data samples except one, thereby ensuring robustness of the privacy guarantee to adversaries with side-information about the database. Moreover differential privacy makes no distributional assumption on the data.

In this work we use the recently proposed notion of mutual information-differential privacy (MI-DP) to analyze the privacy performance of the schemes. This connects to the natural information-theoretic notion of privacy, as well as enabling the use of more standard tools for analysis. Moreover it is shown in \cite{Cuff_MI} that MI-DP directly implies $(\epsilon, \delta)$-differential privacy. %There are two main motivations to study MI-DP. One is to connect the natural notion of privacy studied in information theory community to DP. The other is to also use standard tools of information theory to analyze privacy.

Our contributions are as follows: First, we derive closed-form expressions on the relative objective error achievable by additive noise (Theorem \ref{noisy:thm}) and random projection schemes (Theorem \ref{jlt:thm}), under a privacy constraint, and show that in general random projections achieve better privacy-utility trade-off. We use results from randomized linear algebra \cite{Pilanci} to prove the utility guarantees. Second, using the \mbox{MI-DP} measure, and using the fact that the random projection matrix is private, we make a connection between the MI-DP and SIMO channel, and show that non-coherent SIMO bounds do not give a stronger scaling guarantee than their coherent counterparts. Third, we present numerical results demonstrating the performance of the two schemes. 

\noindent\textbf{Related work.} The works in \cite{Bassily2014, Chaudhuri2011, Kifer2012} propose perturbing the objective to provide privacy guarantees on the trained model, where the training procedure is trusted and has access to the full database, and the adversary can only access the resulting trained model. In contrast, we assume that the training procedure itself may be adversarial, and is not given access to the raw data samples. In the context of linear regression and related problems, the works in \cite{Wasserman2009, Pilanci} propose random projections to provide privacy, by showing that the mutual information between the raw and projected data samples grows sublinearly with dimensions. However, this does not necessarily translate to a formal differential privacy guarantee on the data samples. Random projection as a tool to provide differential privacy has also been considered in \cite{Mironov2012} and \cite{Shiva2016}. The main difference of these works with ours is that they project each data vector individually to a lower-dimensional subspace, whereas we consider mixing samples across the database, such that the effective number of ``mixed" samples is fewer than original.

In terms of motivation and techniques, the works in \cite{Sheffet2012, Sheffet2017} are the most closely related to ours. These works consider JLT in the context of linear regression, and prove that it guarantees differential privacy for well-conditioned data matrices. However, no explicit guarantee on the achievable empirical risk is given. In contrast, we directly analyze the privacy-utility trade-off of additive noise and random projections, where utility is measured by the objective value achieved by the trained model under the privacy scheme, normalized by the true minimum of the objective. We also use the stronger MI-DP privacy\footnote{In \cite{Cuff_MI}, it is shown that for discrete alphabets, the two notions are equivalent; however MI-DP is strictly stronger for continuous alphabets.}, instead of the traditional $(\epsilon, \delta)$-differential privacy. We emphasize that the main novelty of our work lies in the analysis of the algorithms and the resulting theoretical guarantees, and not in the algorithms themselves.

\noindent\textbf{Paper organization.} In Section II we give a brief overview on different privacy metrics followed by the precise problem formulation. Section III includes the main theoretical results of this work. The proof outlines are given in Section IV. Section V gives the numerical results.

\section{Formulation and Background}

In this paper we consider the quadratic optimization 
\begin{align}
\min_{\theta} g( \theta ) := \min_{\theta} \| X \theta - y \|_2^2, \label{opt:main} 
\end{align} where $X \in \mathbb{R}^{n \times d}$ is the data matrix that each row corresponds to one user and $ y \in \mathbb{R}^{n}$ are the response variables. We denote a solution of this optimization problem as $\theta^{\star}$. We use $X_{i,j}$ to denote the $j$-th feature of the $i$-th user data point for $i \in \{1. \cdots, n\}, j \in \{1, \cdots, d\}$. We assume the number of data points is greater than the number of features and $X$ is full column rank.  We assume that $| X_{i,j} | \leq 1$. %We represent vectors and matrices with lower case and capital letter.
Throughout this paper, we use bold letters for random variables to distinguish them from deterministic quantities.

%In this work we aim to preserve the privacy of each user. We briefly review the \emph{differential privacy} that was introduced by Dwork et. al. \cite{Dwork2006}. Throughout this work, we follow \emph{mutual-information} differential privacy as a measure of privacy that was recently introduced by P. Cuff et. al. \cite{Cuff_MI}. To briefly review the notion 
Consider a database $D^N := (D_1, \cdots, D_N)$ that returns a query according to a \emph{randomized} mechanism $q(.)$. Let $D^{-i}$ denote the set of database entries excluding $D_i$.

\begin{definition}[$\epsilon$-mutual-information]
A randomized mechanism $q(.)$ satisfies $\epsilon$-mutual-information (MI-DP) if 
\begin{align}
\sup_{i, P({\bold{D}^N})} I( \bold{D}_i; q(\bold{D}^N) | \bold{D}^{-i}  ) \leq \epsilon \quad \text{bits},  \label{def:MIDP}
\end{align} where the supremum is taken over all distribution on $\bold{D}^N$. 
\end{definition}

We aim to preserve the privacy of each entry of $X$, therefore, in the context of our work, \mbox{$D:= ( X_{1,1}, \cdots, X_{1,d}, X_{2,1}, \cdots, X_{2,d}, \cdots, X_{n,d} )$}.

The notion of $\epsilon$-MI-DP is closely related to \emph{$(\epsilon, \delta)$ differential privacy}\cite{Dwork2010}. We first define the notion of neighbor in databases:

\begin{definition}[Neighbor]
Two databases $D^N$ and $\bar{D}^N$ are called {neighbor} if they differ only in one entry. 
\end{definition}

In the context of our problem, two data matrices are neighbors if they only differ in one entry. Now we are ready to define \emph{$(\epsilon, \delta)$ differential privacy}.

\begin{definition}[$(\epsilon, \delta)$ differential privacy]
A randomized mechanism $q(.)$ satisfies $(\epsilon, \delta)$ differential privacy (DP) if for all neighboring databases $D^N$ and $\bar{D}^N$ and all $S \subseteq \text{Range}( q(.) )$,
\begin{align}
\Pr( q({D}^N) \in S ) \leq e^{\epsilon} \Pr( q( \bar{D}^N ) \in S ) + \delta. \label{def:dp:eps}
\end{align} We say $q(.)$ satisfies $(\delta)$-DP if it satisfies $(0 , \delta)$-differential privacy.
\end{definition} 

Note that neither of MI-DP nor DP impose distributional assumptions on the database and the probabilities arise completely from the randomization of the mechanism.

\begin{proposition}[Theorem 1 in \cite{Cuff_MI}]
$\epsilon$-MI-DP is stronger than $(\epsilon, \delta)$-DP in the sense that for all $\epsilon > 0$ if a mechanism is $\epsilon$-MI-DP, there exists $\epsilon', \delta' $ such that the mechanism satisfies $( \epsilon', \delta' )$-DP. We denote this relation with $\epsilon \text{-MI-DP} \succeq (\epsilon, \delta) \text{-DP}$. Furthermore, we have the following relation:
\begin{align}
\epsilon \text{-MI-DP} \stackrel{(a)}\succeq (\delta) \text{-DP} \stackrel{(b)} \equiv (\epsilon, \delta) \text{-DP},
\end{align}
where $\succeq$ is interpreted as being stronger and $(b)$ means $(\delta) \text{-DP} \succeq (\epsilon, \delta) \text{-DP} $ and $ (\epsilon, \delta) \text{-DP} \succeq (\delta)  \text{-DP}$.
\end{proposition}

\begin{proposition}[See Lemma 2 in \cite{Cuff_MI}]
If a mechanism is $\epsilon\text{-MI-DP}$ then it also satisfies $(0, \sqrt{ \frac{2}{\log(e)}  \epsilon}) \text{-DP}$. \label{prop:MItoDP}
\begin{comment}
\begin{align}
\epsilon\text{-MI-DP} \Longrightarrow (0, \sqrt{ \frac{2}{\log(e)}  \epsilon}) \text{-DP}.
\end{align} \label{prop:MItoDP}
\end{comment}
\end{proposition}

%See \cite{Cuff_MI} for a thorough discussion about different privacy measures and their correspondence.

Let us denote a solution to the original problem \eqref{opt:main} with $\theta^{\star}$. Let us denote the the cost function of the transformed problem with $\hat{g}(\theta)$ with a minimum of $\hat{\theta} \in \argmin_{\theta} \hat{g}(\theta)$. We define the relative error of this transformed problem as the smallest $\eta \geq 1$ such that,
\begin{align}
g( \hat{\theta} ) \leq  \eta g( \theta^{\star} ). \label{relative error}
\end{align}

\begin{comment}
In particular we say $\hat{\theta} := \argmin_{\theta} g(\theta)$ is \emph{$\delta$-optimal approximation} to the original problem \eqref{opt:main} if
\begin{align}
f( \theta^{\star} ) \leq \underbrace{ (1 + \delta) }_{\eta} f(\theta^{\star} ), \label{relative error}
\end{align} where $\eta$ is the \emph{relative error} of this approximation and we use the relative error as the utility.
\end{comment}

%Throughout this paper, we try to answer to this question for the linear regression models:

In this paper we consider the achievable relative error for linear regression given $\epsilon$-MI-DP requirement.

%We analyze two methods for answering this question in Section III and find the privacy-utility trade off. In the rest of this section we define some variables that will be used in the subsequent sections.
\noindent\textbf{Notation.} We denote the \emph{condition number} of $X$ with
\begin{align}
\kappa(X) := \| X \|_2 \| X^{\dagger} \|_2 = \frac{\sigma_{\max}(X) }{\sigma_{\min}(X)},
\end{align} where $X^{\dagger}$ is the Moore-Penrose pseudoinverse of $X$ and $\| X \|_2$ is the spectral norm of $X$.

We denote that ratio of $l_2$ norm of the projection of $y$ onto the column space of $X$ over the $l_2$ norm of the residual with:
\begin{align}
r(y) := \frac{ \| X  \theta^{\star} \|_2 }{ \| X \theta^{\star} - y  \|_2 }. 
\end{align} where $\| X  \theta^{\star} \|_2$ is the $l_2$ norm of the vector $X  \theta^{\star}$.

%In order to give guarantees on the privacy 
We define  $f_i(X) := \sqrt{ \sum\limits_{j=1}^{n} |X_{i,j}^2| - \max_{j} | X_{i,j}^2 | }$ for \mbox{$i \in \{1. \cdots, d\}$}, and \mbox{$f(X) := \min_{i} f_i(X)$}. In order to give guarantees on the privacy of the projection method the amount of additional noise is expressed in terms of $f(X)$.

\section{Privacy-Utility trade off}

In this section we analyze the utility-privacy trade-off for both an additive noise mechanism as well as a scheme with random projection. We compare their utility-guarantees for the same level of $\epsilon$-MI-DP privacy.

%In this section we analyze two randomized mechanism in terms of utility-privacy trade off. First we investigate the naive scheme of adding noise to the data matrix, we derive the amount of noise needed to satisfy $\epsilon$-MI-DP and we give the utility guarantees. In the second scheme, we first encode the data matrix using random projections and we add noise, if necessary to, satisfy $\epsilon$-MI-DP requirement. We derive the utility bound for the second approach and compare it to the additive noise. Throughout this section, we fix the privacy parameter to be $\epsilon > 0$ and we derive bounds for the relative error based on that. Lastly we compare both schemes.

\subsection{Additive Gaussian Noise}

In order to satisfy privacy, we add Gaussian noise directly to the data,
\begin{align}
{X}_{AN}(X) := X + \sigma_{AN} \bold{N},
\end{align} where $\bold{N} \in \mathbb{R}^{n \times d}$ with i.i.d. entries drawn from $\mathcal{N}(0,1)$ and 
\begin{align}
\sigma^2_{AN}(\epsilon) := {\frac{1}{2^{2\epsilon} -1 }}, \label{noisy:var}
\end{align} is the variance of the noise.
\begin{align}
\theta_{AN} &:= \argmin_{\theta} \underbrace{\| {X}_{AN}\theta - y  \|_2^2 }_{g_{AN}(\theta)}  , \label{additive:opt}
\end{align}

\begin{theorem}[Privacy-Utility for Additive Noise]
Given a data set $X$ and {\color{black} the randomized mechanism ${X}_{AN}(X)$ with $\epsilon$-MI-DP constraint}, with probability at least \mbox{$1 - 2e^{ -\frac{ \sigma_{\max}( X )^2}{ 2 \sigma_{AN}^2(\epsilon) } \delta^2 }$} we have the following bound on the relative error {\color{black} of the transformed problem}:
\begin{align}
\eta_{AN} \leq \lp 1 + \frac{ \kappa(X) ( \Delta( X, \epsilon) + \delta ) }{1 - \kappa(X) ( \Delta( X, \epsilon ) + \delta  ) } ( \kappa(X) + r(y) )  \rp^2, \label{noisy:thm:eq}
\end{align} if $\kappa(X) \Delta(\epsilon, X) <1$, where $\Delta(X, \epsilon)= \frac{  \sigma_{AN}(\epsilon) }{ \sigma_{\max}( X ) } \big( \sqrt{n} + \sqrt{d} \big) $ and $\delta > 0$ is a free parameter\footnote{Note that support of $\delta$ is restricted to the set where $ \kappa(X) ( \Delta + \delta  ) \leq 1$  }. \label{noisy:thm}
\end{theorem}

Note that if $\sigma^2_{\max}(X)$ scales linearly with $n$ then $\Delta$ converges to a constant term. Based on Proposition \ref{prop:MItoDP}, additive noise also satisfies $(\delta)$-DP.

%Note that if $\sigma^2_{\max}(X)$ scales linearly with $n$, for example if the data are sampled from an i.i.d. distribution, then $\Delta$ converges to a constant term. Based on Proposition \ref{prop:MItoDP}, additive noise also satisfies $(\delta)$-DP.

\subsection{Gaussian Random Projections }

We encode the data matrix using JLT to a lower dimensional space $n'$ and we add Gaussian noise, when necessary, to guarantee $\epsilon$-MI-DP. We denote the encoded data by ${X}_{RP} \in \mathbb{R}^{ n' \times d  }$  and ${y}_{RP} \in \mathbb{R}^{n'}$:
\begin{align}
{X}_{RP}(X) := \bold{S}X + \sigma_{RP} \bold{N}, \quad {y}_{RP} := \bold{S} y, 
\end{align} where $\bold{S} \in \mathbb{R}^{n' \times n}$ represents the random projection with i.i.d. $\mathcal{N}(0,1)$ entries and $N \in \mathbb{R}^{n' \times d}$ is the noise added to ensure the privacy with i.i.d. entries drawn from $\mathcal{N}(0, 1)$, and 
\begin{align}
\sigma^2_{RP}(X, \epsilon) :=  \big( \frac{n'}{2^{2\epsilon} - 1}  - f^2(X)\big)_{+}, \label{jlt:var}
\end{align}
is the variance of the additive noise\footnote{Note that our algorithm does not reveal this quantity explicitly avoiding an extra privacy leakage.}.

We solve the following problem to estimate the model:
\begin{align}
\theta_{RP} := \argmin_{\theta} \underbrace{\| {X}_{RP}\theta - {y}_{RP}  \|_2^2 }_{g_{RP}(\theta)}  , \label{jlt:opt}
\end{align}

\begin{theorem}[Privacy-Utility for Random Projection]
Given a dataset $X$ and {\color{black}the randomized mechanism ${X}_{RP}(X)$ with $\epsilon$-MI-DP constraint} and a projection dimension of $n' < n$, with probability at least $1 - c_1e^{-c_2 n' \delta^2}$, we have the following bound on the relative error of the transformed problem: \label{jlt:thm}
\begin{align}
\eta_{RP} \leq (1 + \delta)^2 ( 1 +  l_1( X, \epsilon  )   ) ( 1 +  l_2( X, \epsilon  )  )^2, \label{jlt:thm:eq}
\end{align} where {$l_1(X, \epsilon) := \sigma_{RP}^2(X, \epsilon) \lp \max_{i} \frac{\sigma_i(X)}{\sigma_i^2(X) + \sigma_{RP}^2} \rp^2 r^2(y)$},  \mbox{$l_2( X, \epsilon  ) := \frac{\sigma_{RP}^2( X, \epsilon ) }{ \sigma_{\min}^2(X) + \sigma_{RP}^2( X, \epsilon ) }r(y) $}, $\delta \geq \sqrt{c_0 \frac{d}{n'}} $ is a free parameter, and $c_0$, $c_1$ and $c_2$ are constants.
\end{theorem}

\begin{corollary}
The random projection methods also satisfies $(\delta)$-DP for $\delta := \sqrt{ \frac{2}{\log(e)} \epsilon} $.
\end{corollary}

\begin{corollary}
Note that the amount of noise added to the projected data is $\sigma_{RP}^2( \epsilon, X ) =  \big( \frac{n'}{2^{2\epsilon} - 1}  - f^2(X)\big)_{+}$. If $f^2$ scales linearly with $n$ and $n' = o(n)$, asymptotically the noise variance goes to zero, i.e., random projection itself guarantees the privacy. Furthermore, for a given $\delta$, $\eta_{RP} \leq (1 + \delta )^2$ asymptotically as two other terms in \eqref{jlt:thm:eq} vanish. 
\end{corollary}

\begin{remark}
In the proof of Theorem \ref{jlt:thm} in order to derive an upper bound for \eqref{def:MIDP} we make a connection to the SIMO non-coherent channel. We used the coherent SIMO bound for upper bounding this quantity. One may ask if we can get a tighter bound by using  the tighter non-coherent bounds (see for example \cite{Lapidoth}). The known non coherent bound, 
\begin{align}
C \leq \frac{n'}{2} \log(1+ \frac{1}{\sigma_{RP}^2 + f^2(X)}). \label{SIMO:noncoherent} 
\end{align} does not give any improvement. This bound \eqref{SIMO:noncoherent} is known to be tight for the low-SNR regime \cite{Lapidoth}. Therefore when \mbox{$f^2 = \Omega( n )$} asymptotically both bounds yield the same result.
%Therefore when \mbox{$f^2 = \Omega( n )$}, for example data is drawn from an i.i.d. distribution, asymptotically both bounds yield the same result.

\end{remark}

\section{Proof Outlines}
\vspace*{-.2cm}
\subsection{Theorem 1}
\begin{proof}[Proof Outline]
The proof consists of two steps. First we derive the minimum amount of noise needed to ensure $\epsilon$-MI-DP for ${X}_{AN}$ with respect to any feature of users, which is stated in the following lemma:
\begin{lemma}[Privacy Guarantee for the additive noise]
If \mbox{$\sigma_{AN}^2 = \frac{1}{2^{2\epsilon} -1 }$} then ${X}_{AN}$ is $\epsilon$-MI-DP with respect to any entry of $X$. \label{noise:dp}
\end{lemma}
\begin{proof}
We show that \eqref{def:MIDP} is bounded by $\epsilon$ for this choice of $\sigma_{AN}^2$ and $q(X) := {X}_{AN}$. Due to the symmetry of the problem, we fix $\bold{D}_i$ to be the first feature of the first data point without loss of generality. Note that
\begin{align} 
 I( \bold{X}_{1,1}; {\bold{X}}_{AN} | \bold{X}^{-(1,1)} ) = I( \bold{X}_{1,1}; \bold{X}_{1,1} + \sigma_{AN} \bold{N}_{1,1} | \bold{X}^{-(1,1)} ).
\end{align} By expanding the mutual information:
\begin{align}
I( \bold{X}_{1,1}&; \bold{X}_{1,1} + \sigma_{AN} \bold{N}_{1,1} | \bold{X}^{-(1,1)} ) \nonumber \\
&= h(\bold{X}_{1,1} +  \sigma_{AN} \bold{N}_{1,1} | \bold{X}^{-(1,1)}  ) - h( \bold{X}_{1,1} +  \sigma_{AN} \bold{N}_{1,1}  | \bold{X}  ) \nonumber\\
&\stackrel{(a)}= h(\bold{X}_{1,1} +  \sigma_{AN} \bold{N}_{1,1} | \bold{X}^{-(1,1)}  ) - h( \sigma_{AN} \bold{N}_{1,1}   ),
 \label{noise:dp:proof}
\end{align} where $(a)$ holds because the noise is independent of the data. Now we need to take the maximization over all possible distribution on $\bold{X}$. Note that the absolute value of each entry is bounded by $1$ therefore we need to take the supremum over all distribution inside this ball. The absolute value constraint implies the second moment constraint for all distribution defined on it, therefore by using the maximum entropy bound the result follows:
\begin{align}
\sup_{ P({\bold{X}})} I( \bold{X}_{1,1}; {\bold{X}}_{AN} | \bold{X}^{-(1,1)}  ) \leq \frac{1}{2} \log(1 + \frac{1}{\sigma_{AN}^2} ) = \epsilon.   \label{noisy:dp:proof}
\end{align} 

\end{proof}
\vspace{-.5cm}
%Note that $ X \leftrightarrow {X}_{AN} \leftrightarrow \theta_{AN}$ forms a Markov chain therefore $\theta_{AN}( X )$ is also $\epsilon$-MI-DP. 
The second step bounds the relative error. We use perturbation theory in the least square setup (see Theorem 5.1 in \cite{Wedin}) and probabilistic bounds on the maximum singular value of an i.i.d. Gaussian to derive the result \cite{Rudelson}. The details of the proof are provided in \cite{Showkatbakhsh2018}.
\end{proof}

\vspace*{-.3cm}

\subsection{Theorem 2}
\begin{proof}[Proof Outline]
The proof consists of two steps. First we find the variance of noise needed to add to satisfy $\epsilon$-MI-DP that results to $\epsilon$-MI-DP model, $\theta_{RP}$. Following lemma characterizes the amount of noise sufficient to make the mechanism  $\epsilon$-MI-DP.

\vspace*{-.2cm}

\begin{lemma}
If $\sigma_{RP}^2 := ( \frac{n'}{2^{2\epsilon}-1} - f^2(X) )_{+}$ then ${X}_{RP}$ is $\epsilon$-MI-DP with respect to any entry of $X$. \label{jlt:dp:thm}
\end{lemma}

\vspace*{-.3cm}

\begin{proof}
We show that the conditional mutual information \eqref{def:MIDP} is bounded by $\epsilon$ for this choice of $\sigma_{RP}^2$. Due to the symmetry of the problem, we fix $D_i$ to be the first feature of the first data point.
\begin{align}
&\max_{ P(\bold{X}) \in \mathbb{P} } I( \bold{X}_{1,1};  {\bold{X}}_{RP} | \bold{X}^{-(1,1) } ) \label{jlt:dp:max} \\
  &= \max_{ P(\bold{X}^{-(1,1)} ) p(\bold{X}_{1,1} | {X}^{-(1,1)} )  } \mathbb{E}_{ \bold{X}^{-(1,1)} } [ I( \bold{X}_{1,1}; {\bold{X}}_{RP} | \bold{X}^{-(1,1)}  = X^{-(1,1)} ) ], \nonumber\\  
& \stackrel{(a)}{\leq} \max_{ P(\bold{X}^{-(1,1)} ) } \mathbb{E}_{\bold{X}^{-(1,1)}} [  \max_{ P(\bold{X}_{1,1} | X^{-(1,1)} ) } I(  \bold{X}_{1,1}; {\bold{X}}_{RP} | \bold{X}^{-(1,1)}  = X^{-(1,1)}  )  ], \nonumber
\end{align} where $\mathbb{P}$ is the set of distributions which assign non-zero measure to $\bold{X}$ only if the absolute value of each entry {\color{black} is upper bounded by $1$ and $f(\bold{X})$ is lower bounded by the $f(.)$ evaluated for the original database}, $(a)$ follows from the Jensen's Inequality and the fact that maximization over the conditional distribution is a convex function. Now we find upper bounds on the the inside of the expectation, note that columns of ${\bold{X}}_{RP}$ rather than first one does not have any term associated with $\bold{X}_{1,1}$ and they are conditionally independent, therefore we can write
\begin{align}
& \max_{ P(\bold{X}_{1,1} | X^{-(1,1)} ) } I(  \bold{X}_{1,1}; {\bold{X}}_{RP} | \bold{X}^{-(1,1)}  = X^{-(1,1)}  ) \nonumber \\ 
&= \underbrace{ \max_{P(\bold{X}_{1,1} | X^{-(1,1)} )  } I( \bold{S} \bold{X}^{(:,1)} + \sigma_{RP} \bold{N}^{(:, 1)} ; \bold{X}_{1, 1} |\bold{X}^{-(1,1)}  = X^{-(1,1)}   ) }_{(\star)}, \nonumber
\end{align} where $X^{(:,1)}$ denotes the first column of $X$. We find an upper bound on $(\star)$ for a fixed set of $X^{-(1,1)}$ We observe that $(\star)$ is same as the capacity of the following non-coherent SIMO channel with Rayleigh fading with a unit power constraint:
\begin{align}
z = \bold{S}^{(:,1)} \bold{X}_{1,1} + \underbrace{\sum\limits_{i \neq 1} \bold{S}^{(:,i)} {X}_{i,1} + \sigma_{RP} \bold{N}^{(:,1)} }_{\nu}, \label{SIMO}
\end{align} where $z \in \mathbb{R}^{n'}$ is the first column of ${\bold{X}}_{RP}$ which we treat here as the output of the channel. Note that ${X}_{i,1}$ ( $i \neq 1$ ) are treated as constants for this channel and therefore $\nu$ is effectively a zero mean i.i.d. Gaussian noise with the covariance of
\begin{align}
\mathbb{E}[\nu \nu^T] = ( \sigma_{RP}^{2}  +  \sum\limits_{i \neq 1} ( {X}_{i,1} )^2 )  I_{n'} = \sigma_{\nu}^2 I_{n'}, \label{SIMO:var}   % \mathbb{E}[   \bold{S}^{(:,i)}  { \bold{S}^{(:,i)} }^T] = 
\end{align} Now we bound the capacity of this channel, We use the coherent upper bound for the capacity of this channel:
\begin{align}
\max_{P(\bold{X}_{1,1} )} I( \bold{X}_{1,1} ; z )& \leq \max_{P(\bold{X}_{1,1}^1)} I( \bold{X}_{1,1}; z, \bold{S}^{(:,i)}  ) \nonumber \\
&\stackrel{(a)}\leq \mathbb{E}_{\bold{S}^{(:,i)} }[ \frac{1}{2} \log( 1 + \frac{ \| \bold{S}^{(:,i)}  \|^2 }{\sigma_{\nu}^2   } ) ] \nonumber \\
& \stackrel{(b)}\leq \frac{1}{2} \log( 1 + \mathbb{E}_{ \bold{S}^{(:,i)}}[ \frac{ \| \bold{S}^{(:,i)}  \|^2 }{f(X)^2 + \sigma_{RP}^2   } ] ) \stackrel{(c)}\leq \epsilon. 
\end{align} Note that the absolute value constraint implies the second moment constraint for all distribution defined on it and $(a)$ follows from the maximum entropy bound, $(b)$ follows directly from the Jensen's Inequality, $(c)$ comes from the fact that the outer maximization is over distributions that assign non-zero measure to $\bold{X}$ only if \mbox{$f( \bold{X} ) \geq f(X)$}.
\end{proof}

Now we derive the utility guarantee for this method. Note that by rewriting ${X}_{RP} = \begin{bmatrix} S & N \end{bmatrix} \begin{bmatrix} X \\ \sigma_{RP} I \end{bmatrix} = \tilde{S} \begin{bmatrix} X \\ \sigma_{RP} I \end{bmatrix} $ we observe that adding direct noise to the projected data can be interpreted as the random projection of the $l_2$ regularized least square problem (Ridge Regression), i.e.,
\begin{align}
\theta_{RP} &= \argmin_{\theta} \| {X}_{RP} \theta - {y}_{RP} \|_2^2 \\ 
&= \argmin_{\theta} \| \tilde{S} \underbrace{\big( \begin{bmatrix} X \\ \sigma_{RP} I \end{bmatrix} \theta - \begin{bmatrix} y \\ 0  \end{bmatrix}  \big)}_{\text{RR}} \|_2^2. 
\end{align} Therefore we can split the utility analysis into two parts,
\begin{enumerate}
\item What is the utility loss for the $l_2$ regularized least square? 
\item What is the utility loss for the randomized sketching (JLT)?
\end{enumerate}
We use the standard SVD argument to bound the Ridge Regression relative error and by following Pilanci et. al. \cite{Pilanci} (see Corollary 2) we give guarantees on the sketching step. The details of the proof are provided in \cite{Showkatbakhsh2018}.
\end{proof}

\section{Numerical Results}

We numerically evaluate the relative error $\eta$ achieved by the schemes in Section III subject to an $\epsilon$-MI-DP constraint. %We implement the privacy schemes as described in Section III and solve the resulting optimization problems.

\begin{figure}
\centering
\includegraphics[scale=0.26]{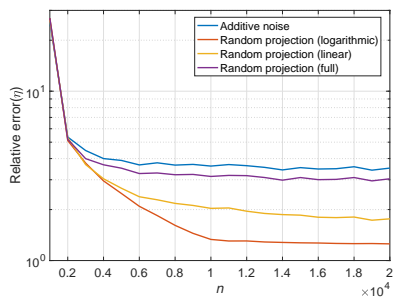}
\caption{Relative error of the schemes for $\epsilon=0.5$, for random data}
\label{fig:random_05}
\end{figure}

\begin{figure}
\centering
\includegraphics[scale=0.28]{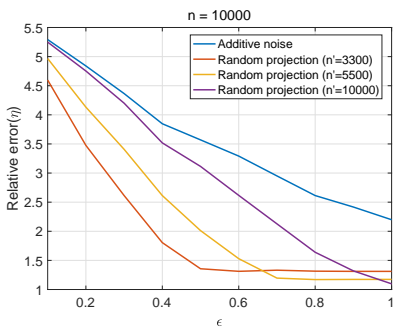}
\caption{Relative error vs. $\epsilon$, for $n=10000$, for random data}
\label{fig:eps_sweep}
\end{figure}

\subsection{Random data}
We generate the elements $X_{i,j}$ i.i.d. uniformly in the interval $[-1,1]$, where $X \in \mathbb{R}^{n \times 800}$, and $n = 1000k$ with $k \in \{ 1,2,\dots,20\}$. For each case, the additive noise parameter $\sigma_{AN}$ is computed according to \eqref{noisy:var}. Similarly, the additive noise $\sigma_{RP}$ is computed according to \eqref{jlt:var}. Given $k$, we evaluate three choices of $n'$: logarithmic ($n'_1 := 1000\lp\log\lp k\rp+1\rp$), linear ($n'_2 :=1000\frac{k+1}{2}$), and full ($n' = n = 1000k$). The resulting relative error curves are given in Figure~\ref{fig:random_05} for $\epsilon=0.5$, averaged over 5 trials. We note that random projection results in uniformly better privacy-utility trade-off compared to additive noise. Further, at this regime of $\epsilon$, lower projection dimensions result in significantly better trade-off. Figure~\ref{fig:eps_sweep} plots the achieved relative error as a function of $\epsilon$, for $n=10000$. We note that the relative error decreases linearly until it saturates for all schemes, and for stricter privacy constraints (small $\epsilon$), lower projection dimension achieves smaller relative error. As $\epsilon$ tends higher, the privacy constraint becomes less restrictive, and schemes with higher projection dimension perform better because of the additional rows of information.

%We note that for a modest privacy constraint $\epsilon=0.5$, all random projection schemes strictly dominate additive noise in terms of relative error. As $n$ increases, we note that schemes with higher projection dimension lead to a better privacy-utility trade-off than that for lower dimension. Inspecting \eqref{jlt:thm:eq}, we see that this is because when $\epsilon$ is higher, the additive component of the noise $\sigma_{RP}$ grows slowly with $n'$, thereby the benefit of the extra rows of data dominate the harm of additive noise from the utility perspective. On the other hand, for a tighter privacy constraint $\epsilon=0.2$, we see that while random projection with logarithmic dimension still performs well compared to additive noise scheme, high projection dimensions end up performing worse than additive noise. This is because as a result of the linear mixing of data, all projected rows of $\tilde X$ reveal information about the single element $X_{ij}$, and thus for high projection dimensions, a high degree of noise is needed to compensate the effect, leading to degradation of the utility.

\begin{figure}
\centering
\includegraphics[scale=0.25]{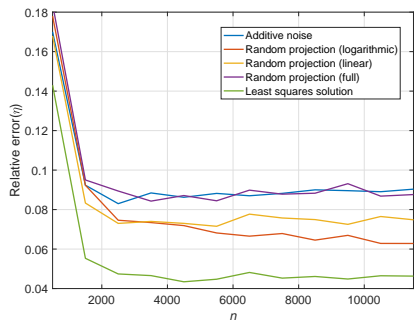}
\caption{Test error of the schemes for $\epsilon=0.2$, for MNIST}
\label{fig:mnist02}
\end{figure}

\subsection{MNIST Handwritten Digits Dataset}
We consider a reduced version of the MNIST hand-written digits dataset \cite{mnist}, where we only take the digits 4 and 9, leading to 11791 data samples, and only consider the 300 pixels that contain the most energy across these data samples. Mapping the digit labels to $+1$ and $-1$, and vectorizing each data image, we solve the corresponding linear problem, which generates a model that classifies 4's versus 9's. Figure~\ref{fig:mnist02} gives the resulting test error (subject to a 80\%/20\% training/test set partition) for the logarithmic ($n'_1 := 500\lp\log\lp k\rp+1\rp$), linear ($n'_2 :=500\frac{k+1}{2}$), and full ($n' = n = 1000k$) random projections, as well as additive noise. To evaluate values of $n$ smaller than 11791, we randomly sample the dataset. The results are averaged over 10 trials. Similar to the random case, we observe that random projection with logarithmic dimensions result in the best performance, while preserving MI-DP with $\epsilon=0.2$.

\bibliographystyle{ieeetr}
\bibliography{references}

\appendices
\section{Additive Gaussian Noise}

In order to derive bounds for the utility performance of additive noise, we use the perturbation theory in the least square setup \cite{Wedin}.
For a given $N$ and $\sigma_{AN}$ we have the following deterministic bound on the utility, 
\begin{lemma}[See Theorem 5.1 in \cite{Wedin}] 
Assuming $\rank(X) = \rank( X + \sigma_{AN} N ) = d$ and $ \kappa(X) \Delta(\epsilon,N,X) < 1$:
 \begin{align}
 \frac{\| X \theta_{AN} - y  \|_2}{\| X \theta^{\star} - y  \|_2} \leq 1 + \frac{ \kappa(X) \Delta(\epsilon,N,X) }{1 - \kappa(X) \Delta(\epsilon,N,X) } ( \kappa(X) + r(y) ), \label{noisy:utility:eq:noise}
 \end{align} where $\Delta(\epsilon, N, X)  = \sigma_{AN} \frac{ \| N \|_2}{\| X \|_2}$. \label{noisy:lemma:purt}
\end{lemma} It is well-known that the maximum singular value of $N \in \mathbb{R}^{n \times d}$ converges almost surely to $\sqrt{n} + \sqrt{d}$ asymptotically. For the non-asymptotic bounds, we use the following lemma:

\begin{proposition}[See \cite{Rudelson}]
If $N \in \mathbb{R}^{n \times d}$ is a Gaussian random matrix with entries drawn from $\mathcal{N}(0, 1)$, then
\begin{align}
%P(  \sqrt{n} -  \sqrt{d} - t   \leq s_{\min}( N ) & \leq s_{\max}(N) \leq \sqrt{n} + \sqrt{d} + t )  \\
P(  \sigma_{\max}(N) \leq \sqrt{n} + \sqrt{d} + t )  \geq 1-2e^{-\frac{t^2}{2}}, \quad t \geq 0. \label{noisy:sbound}
\end{align} \label{noisy:prop:sbound}
\end{proposition} By combining Lemma \ref{noisy:lemma:purt} and Proposition \ref{noisy:prop:sbound} and the choice of $\sigma_{AN}$ \eqref{noisy:utility:eq:noise}, Theorem \ref{noisy:thm} directly follows.

\section{Gaussian Random Projections}

In this section, we derive utility guarantee on the performance of random projection for the given value of $\sigma_{RP}$. Note that by rewriting ${X}_{RP} = \begin{bmatrix} S & N \end{bmatrix} \begin{bmatrix} X \\ \sigma_{RP} I \end{bmatrix} = \tilde{S} \begin{bmatrix} X \\ \sigma_{RP} I \end{bmatrix} $ we observe that adding direct noise to the projected data can be interpreted as the random projection of the $l_2$ regularized least square problem (Ridge Regression), i.e.,
\begin{align}
\theta_{RP} &= \argmin_{\theta} \| {X}_{RP} \theta - {y}_{RP} \|_2^2 \\ 
&= \argmin_{\theta} \| \tilde{S} \underbrace{\big( \begin{bmatrix} X \\ \sigma_{RP} I \end{bmatrix} \theta - \begin{bmatrix} y \\ 0  \end{bmatrix}  \big)}_{\text{RR}} \|_2^2,
\end{align} Let us denote the solution to the Ridge Regression problem with $\theta_{RR} = \argmin_{\theta} \| X \theta - y \|^2 + \sigma_{RP}^2 \| \theta \|^2$, therefore we can write:
\begin{align}
\frac{\| X \theta_{RP} - y \|_2^2}{\| X \theta^{\star} - y \|_2^2} &=  \underbrace{\frac{\| X \theta_{RR} - y \|_2^2}{\| X \theta^{\star} - y \|_2^2}}_{\eta_1}   \nonumber \\
& \times  \underbrace{\frac{\| X \theta_{RR} - y \|_2^2 + \sigma_{RP}^2 \| \theta_{RR} \|_2^2}{\| X \theta_{RR} - y \|_2^2}}_{\eta_2}   \nonumber  \\ 
& \times  \underbrace{ \frac{\| X \theta_{RP} - y \|_2^2 + \sigma_{RP}^2 \| \theta_{RP} \|_2^2}{\| X \theta_{RR} - y \|_2^2 + \sigma_{RP}^2 \| \theta_{RR} \|_2^2} }_{\eta_3} \nonumber \\
& \times \underbrace{ \frac{\| X \theta_{RP} - y  \|_2^2}{\| X \theta_{RP} - y \|_2^2 + \sigma_{RP}^2 \| \theta_{RP} \|_2^2} }_{\eta_4}.
\end{align} It is clear that $\eta_4 < 1$, therefore we find bounds for each of $\eta_1$, $\eta_2$ and $\eta_3$.

Using the following Lemma, $\eta_1 \leq \lp 1 + \frac{\sigma^2_{RP}}{ \sigma^2_{\text{min}} + \sigma^2_{RP} } r(y) \rp^2$.

\begin{lemma}
Let us denote the solution to the $l_2$ regularized least square problem with $\theta_{RR}(\lambda) := \argmin_{\theta} \|X\theta - y\|_2^2 + \lambda\| \theta \|_2^2$ and $\theta^{\star} = \argmin_{\theta} \| X \theta - y \|_2^2$, then we have the following bound on the empirical risk loss given that $X$ is full rank:
\begin{align}
\frac{\| X \theta_{RR}(\lambda) - y  \|_2}{ \| X \theta^{\star} - y  \|_2}
& \leq 1 + \frac{\lambda}{ \sigma^2_{\text{min}} + \lambda } r(y). \label{jlt:RRlemma:eq}
\end{align} \label{jlt:RRlemma}
\end{lemma}

\begin{proof}
Using triangle inequality we can write the LHS of \eqref{jlt:RRlemma:eq}:
\begin{align}
\frac{\| X \theta^{\star} - y  +  X ( \theta_{RR}(\lambda) - \theta^{\star} )  \|_2}{\| X \theta^{\star} - y  \|_2} \leq 
1 + \frac{\| X ( \theta_{RR}(\lambda) - \theta^{\star} ) \|_2}{\| X \theta^{\star} - y  \|_2}. \label{jlt:RRlemma:proof:eq1}
\end{align} Let us denote the SVD decomposition of $X$ by $X = U \Sigma V^T$, where $U \in \mathbb{R}^{n \times d}$ spans the column space, $\Sigma \in \mathbb{R}^{d \times d}$ is the diagonal matrix of the singular values and $V \in \mathbb{R}^{d \times d}$ spans the row space of $X$. We use the close form solution for $\theta^{\star}$ and $\theta_{RR}$ to derive bounds for $\| X ( \theta_{RR}(\lambda) - w^{\star} ) \|_2$.
\begin{align}
\theta^{\star} &= (X^T X )^{-1} X^T y = V \Sigma^{-1} U^T y \\
\theta_{RR} &= (X^T X + \lambda I )^{-1} X^T y = V( \Sigma^2 + \lambda I )^{-1} \Sigma U^T y,
\end{align} therefore
\begin{align}
\| X ( \theta_{RR}(\lambda) - \theta^{\star} ) \|_2 &= \| U \underbrace{\Sigma [ ( \Sigma^2 + \lambda I )^{-1} - \Sigma^{-2}  ] \Sigma}_{- D} U^T y  \|_2 \nonumber \\
& \leq  \| U D U^T y \|_2 \nonumber \\
& \stackrel{(a)}\leq \sigma_{\max}( D ) \| U^{T} y \|_2 \nonumber \\
& \stackrel{(b)}= \lp \frac{\lambda}{ \sigma^2_{\min} + \lambda}  \rp \| X \theta^{\star} \|_2. \label{jlt:RRlemma:proof:eq2}
\end{align} $(a)$ and $(b)$ follow directly since  $D$ is a diagonal matrix with $i$-th entry of $\frac{\lambda}{ \sigma^2_i + \lambda}$, where $\sigma_i$ is $i$-th singular value of $X$ so $\sigma_{\max}( D ) \leq \frac{\lambda}{ \sigma^2_{\min}+ \lambda}$.
By combining \eqref{jlt:RRlemma:proof:eq1} and \eqref{jlt:RRlemma:proof:eq2}, \eqref{jlt:RRlemma:eq} follows directly.
\end{proof}

\begin{corollary}
Let us denote the solution to the $l_2$ regularized least square problem with $\theta_{RR}(\lambda) := \argmin_{\theta} \|X\theta - y\|_2^2 + \lambda\| \theta \|_2^2$ and $\theta^{\star} = \argmin_{\theta} \| X \theta - y \|_2^2$, we have the following bound on the norm of the $\theta_{RR}$: \label{jlt:RRcor}
\begin{align}
{\| \theta_{RR} \|_2} \leq \lp \max_{i} \frac{\sigma_i}{\sigma_i^2 + \lambda} \rp \| X \theta^{\star} - y \|_2,
\end{align}
\end{corollary}
\begin{proof}
Proof directly follows by using the closed form solution for $\theta_{RR}$, 
\begin{align}
\|  \theta_{RR} \| &= \| V ( \Sigma^2 + \lambda I )^{-1} \Sigma U^T y \|_2 \nonumber \\
&=  \| \underbrace{( \Sigma^2 + \lambda I )^{-1} }_{D'} \Sigma U^T y \|_2, \nonumber \\
&\leq \sigma_{\max}( D' ) \| U^T y \|_2 \\
&= \lp \max_{i} \frac{\sigma_i}{\sigma_i^2 + \lambda} \rp \| X \theta^{\star} \|_2.
\end{align}
\end{proof}

By Corollary \ref{jlt:RRcor} we have the following bound on $\eta_2$:
\begin{align}
\eta_2 \leq 1 + \sigma_{RP}^2 \lp \max_{i} \frac{\sigma_i}{\sigma_i^2 + \sigma_{RP}^2} \rp^2 r^2(y),
\end{align}

We use results of Pilanci et. al. \cite{Pilanci} for $\eta_3$:

\begin{proposition}[See Corollary 2 in \cite{Pilanci}]
Suppose \mbox{ $ \theta_{RP} = \argmin_{\theta} \| \tilde{\bold{S}} \big( \begin{bmatrix} X \\ \sigma_{RP} I \end{bmatrix} \theta - \begin{bmatrix} y \\ 0  \end{bmatrix}  \big) \|_2^2$} and \mbox{$\theta_{RR} := \argmin_{\theta} \|X\theta - y\|_2^2 + \sigma_{RP}^2 \| \theta \|_2^2$} where $\tilde{\bold{S}} \in \mathbb{R}^{n' \times n}$ is a random Gaussian matrix with entries drawn from $\mathcal{N}(0,1)$. With probability at least $1 - c_1 e^{-c_2 n' \delta^2}$ for $\delta \geq \sqrt{c_0 \frac{d}{n'}}$:
\begin{align}
\frac{\| X \theta_{RP} - y \|_2^2 + \sigma_{RP}^2 \| \theta_{RP} \|_2^2}{\| X \theta_{RR} - y \|_2^2 + \sigma_{RP}^2 \| \theta_{RR} \|_2^2}   & \leq ( 1 + \delta  )^2,
\end{align} where $c_0$, $c_1$ and $c_2$ are constants.
\end{proposition}

Result of Theorem \ref{jlt:thm} follows directly by using bounds on $\eta_1$, $\eta_2$ and $\eta_3$.

\begin{comment}
\begin{remark}
We further investigate the capacity of Channel \ref{SIMO} by using the tighter non-coherent bounds. In \cite{Lapidoth} the authors derive bounds for the capacity of a non coherent channel by optimizing over the output distribution. For this channel the analysis yield the following bound:
\begin{align}
C \leq & \frac{1}{2}\inf_{\alpha, \beta' >0} \inf_{\delta' \geq 0 } \big\{ -n' + n' \Psi(n') - \log\Gamma(n') + \alpha \log(\beta')  \nonumber \\
& + \log\Gamma( \alpha, \frac{\delta'}{\beta'} ) - \alpha \psi(n') + \frac{n'(1+SNR)}{\beta'} + \frac{\delta'}{\beta'} \nonumber \\
&+ \frac{1- \alpha}{n' - 1}\delta'I\{ \alpha \leq 1 \} \big\}, \label{jlt:noncoherent}
\end{align} where $SNR = \frac{1}{\sigma^2_{\nu}}$. In the privacy setup, the goal is to keep this capacity small and we are operating in the low $SNR$ regime. In this regime, the optimal value of $\alpha$ and $\delta$ tends to $n'$ and $0$, respectively \cite{Lapidoth}. Furthermore the optimal value of $\beta'$ tends to $(1+SNR)$. If we use these values in RHS of \eqref{jlt:noncoherent},
\begin{align}
C \leq \frac{n'}{2} \log(1+ \frac{1}{\sigma_\nu^2}). 
\end{align}  We observe that this bound and the coherent upper bound of $\frac{1}{2} \log( 1 + \frac{n'}{\sigma_\nu^2} )$ scales similarly in the low $SNR$ regime.
\end{remark}
\end{comment}

%\section*{Acknowledgment}

%%%%%%
%% To balance the columns at the last page of the paper use this
%% command:
%%
%\enlargethispage{-1.2cm} 
%%
%% If the balancing should occur in the middle of the references, use
%% the following trigger:
%%
\IEEEtriggeratref{3}
%%
%% which triggers a \newpage (i.e., new column) just before the given
%% reference number. Note that you need to adapt this if you modify
%% the paper.  The "triggered" command can be changed if desired:
%%
%\IEEEtriggercmd{\enlargethispage{-20cm}}
%%
%%%%%%

%%%%%%
%% References:
%% We recommend the usage of BibTeX:
%%
%\bibliographystyle{IEEEtran}
%\bibliography{definitions,bibliofile}
%%
%% where we here have assume the existence of the files
%% definitions.bib and bibliofile.bib.
%% BibTeX documentation can be obtained at:
%% http://www.ctan.org/tex-archive/biblio/bibtex/contrib/doc/
%%%%%%

\end{document}